\documentclass[11pt,a4paper]{article}

\usepackage[margin=1in]{geometry}

\usepackage{amsmath}
\usepackage{amssymb}
\usepackage{amsfonts}
\usepackage{amsthm}
\usepackage{mathtools}

\usepackage{bm}
\usepackage{physics}

\usepackage{graphicx}
\usepackage{subcaption}
\usepackage{booktabs}
\usepackage{array}
\usepackage{placeins}
\usepackage{appendix}
\usepackage{hyperref}
\usepackage[nameinlink,capitalize,noabbrev]{cleveref}

\usepackage{microtype}
\usepackage[backend=biber, style=numeric]{biblatex}
\usepackage{cleveref}

\theoremstyle{definition}
\newtheorem{definition}{Definition}[section]

\theoremstyle{plain}
\newtheorem{theorem}{Theorem}[section]
\newtheorem{lemma}{Lemma}
\newtheorem{proposition}{Proposition}[section]

\theoremstyle{remark}
\newtheorem{remark}{Remark}[section]

\newcommand{\R}{\mathbb{R}}

\newcommand{\D}{\mathbb{D}}

\newcommand{\fnorm}[1]{\left\lVert #1 \right\rVert}
\newcommand{\fabs}[1]{\left\lvert #1 \right\rvert}

\title{\textbf{Notes on Fourier-Bessel Wavelets}\\
\large Theory, Construction, and Fourier-Domain Representation}

\author{
  Marcel Venturotti\thanks{Department of Computer Science, University of Bath. Email: \texttt{mv514@bath.ac.uk}} 
  \and 
  Georgios Exarchakis\thanks{Department of Computer Science, University of Bath. Email: \texttt{ge394@bath.ac.uk}}
}

\begin{document}

\maketitle

\begin{abstract}
These notes develop the mathematical foundations and construction of a Fourier-Bessel wavelet family inspired by the disk harmonics of Shaqfa et al.\cite{Shaqfa_2025}. We begin with the relevant properties of Bessel and modified Bessel functions and introduce the wavelet properties required for the construction. We then derive the Fourier-Bessel disk harmonics as solutions to the Helmholtz equation on the unit disk subject to a Neumann boundary condition.

Building on this basis, we construct a wavelet family by applying a Gaussian spatial envelope and introducing a zero-mean correction for the zeroth angular order. We derive the corresponding normalisation constants for $L^2$-based applications and discuss $L^1$-based normalisation for frequency-domain peak consistency. Finally, we derive a closed-form Fourier-domain representation of the resulting wavelets.

The main motivation is the approximately linear spacing, which converges to $\pi$ between consecutive radial eigenvalues. Rather than replacing the conventional dyadic organisation of wavelet families, this construction lays out the foundation to explore whether a more uniform radial frequency allocation can be useful for applications in which broad and balanced frequency coverage is desirable. 

\end{abstract}

\newpage
\tableofcontents
\newpage
\section*{Notation}

\begin{table}[h]
    \centering
    \begin{tabular}{>{$}c<{$} p{0.68\textwidth}}
        \toprule
        \textbf{Symbol} & \textbf{Meaning} \\
        \midrule
        \rho,\varphi
        & Radial and angular polar coordinates in the spatial domain. \\

        x,y
        & Cartesian spatial coordinates, with
          $x=\rho\cos\varphi$ and $y=\rho\sin\varphi$. \\

        q,\phi
        & Radial and angular polar coordinates in the frequency domain. \\

        m
        & Angular order of the Fourier-Bessel function. \\

        k
        & Root index, corresponding to the $k$-th positive root of $J_m'$. \\

        \lambda_{m,k}
        & The $k$-th positive Neumann eigenvalue for angular order $m$,
          satisfying $J_m'(\lambda_{m,k})=0$. \\

        J_m
        & Bessel functions of the first kind.\\

        I_m
        & Modified Bessel functions of the first kind. \\

        N_{m,k}
        & Normalisation constant for the Fourier-Bessel disk basis. \\

        K_{m,k}
        & Zero-mean correction term, non-zero only for $m=0$. \\

        N^{(2)}_{m,k}
        & $L^2$ normalisation constant for the wavelet. \\

        N^{(1)}_{m,k}
        & Peak normalisation constant, obtained from the radial Fourier
          response. \\

        \psi_{m,k}
        & Fourier-Bessel wavelet in the spatial domain. \\

        \widehat{\psi}_{m,k}
        & Fourier-Bessel wavelet in the frequency domain. \\

        q_{m,k}^\ast
        & Frequency at which the radial Fourier response attains its
          maximum. \\
        \bottomrule
    \end{tabular}
    \caption{Notation used throughout the notes.}
\end{table}

\newpage

\section{Introduction}

Solid harmonic wavelets \cite{Eickenberg_2018}
provided an early example of constructing wavelets from solutions to differential equations. By using harmonic functions, which are solutions to Laplace's equation, the authors obtained wavelets whose Fourier representations form approximately ring-shaped structures rather than the more traditional Gaussian-shaped responses. This provides controlled coverage of frequency space with tunable overlap. Nevertheless, the resulting filter banks still rely on dyadic scaling and rotation, which produces a frequency organisation that places progressively greater separation between higher frequency bands.

Wavelets provide a natural framework for analysing signals at multiple scales, decomposing a signal into localised oscillatory components indexed jointly by position and scale \cite{mallat_book,daubechies}. The multiresolution analysis introduced by Mallat \cite{multires} formalised this idea by organising signal information into a hierarchy of nested approximation spaces linked by dyadic dilations, giving wavelet decompositions their characteristic logarithmic frequency tiling. This dyadic organisation, in which each scale doubles the previous one, underlies much of classical wavelet theory, including results on regularity, sparsity, and stability. More recently, wavelet filter banks have been used as the basis for scattering networks \cite{bruna2012invariantscatteringconvolutionnetworks, sifre2014rigidmotionscatteringtextureclassification}, which cascade wavelet transforms with pointwise nonlinearities and averaging operators to build signal representations that are stable to deformations while retaining high-frequency information that is otherwise lost under simple averaging. These constructions typically inherit the dyadic scale structure of classical wavelets, motivating interest in alternative filter constructions that depart from this scaling while retaining wavelet-like localisation properties.

In this work, we explore a different construction inspired by the disk harmonics introduced by Shaqfa et al.~\cite{Shaqfa_2025}. Rather than starting from solutions to Laplace's equation, we use the Bessel-function solutions of the Helmholtz equation on the unit disk. These functions provide radial oscillations whose corresponding eigenvalues become approximately linearly spaced, with asymptotic spacing $\pi$. We use these eigenfunctions as the oscillatory component of a new wavelet family.

The resulting Fourier-Bessel wavelets share the ring-like frequency structure of solid harmonic wavelets, while replacing the traditional dyadic scale parameter with the eigenvalue associated with the radial Bessel function. This provides a natural mechanism for controlling the radial frequency location of the filters.

Importantly, the use of approximately linearly spaced radial frequencies is not proposed as a replacement for dyadic scaling. Dyadic scaling is fundamental to much of wavelet theory and underlies important theoretical properties such as Lipschitz continuity and stability under diffeomorphisms. These properties have not been established for the present construction. Rather, the motivation here is to propose an alternative allowing to explore whether a different frequency organisation can be useful in settings where approximately uniform representation of radial frequencies is desirable. In particular, this may be relevant to reconstruction oriented tasks, where uniform frequency coverage can be preferable to deliberately allocating greater representation to lower frequencies.

\paragraph{Scope of these notes.}
The purpose of this manuscript is to provide a detailed and self-contained derivation of the Fourier-Bessel wavelet construction introduced here. Rather than presenting a complete empirical evaluation, we focus on the mathematical motivation, construction, normalisation, and Fourier-domain
representation of the proposed wavelets. We also provide a small number of numerical examples illustrating their frequency-domain behaviour and frame like coverage. 

\newpage
\paragraph{Contributions.}
The main novel mathematical components developed in these notes are:
\begin{enumerate}

    \item the construction of a Gaussian-windowed Fourier-Bessel wavelet
    family together with the zero-mean correction required for the
    $m=0$ mode.

    \item closed-form expressions for the $L^1$ and $L^2$ normalisation constants
    and the corresponding frequency-domain peak normalisation.

    \item a closed form Fourier domain representation of the resulting
    wavelets.

    \item preliminary numerical evidence that the linear eigenvalue spacing can provide more uniform frequency coverage than Solid Harmonics.
\end{enumerate}

To accompany the mathematical development, we have implemented the construction in the Python library (\texttt{fbscatnet}\footnote{Code available at: \url{https://github.com/Smee18/FourierBesselWavelets}}). The library is intended to make the derivations and figures in these notes reproducible and provides functionality for using the wavelets within a scattering
network, in a similar manner to Kymatio~\cite{andreux2022kymatioscatteringtransformspython}.

\FloatBarrier
\section{Mathematical Background}

\subsection{Bessel Functions}

Bessel functions arise naturally when solving the radial component of the Helmholtz equation in polar coordinates. The Bessel equation of order $m$ is

\begin{equation}
\label{eq:bessel}
x^2\frac{d^2y}{dx^2} +x\frac{dy}{dx} +(x^2-m^2)y=0.
\end{equation}

Because this is a second-order ordinary differential equation, it has two linearly independent solutions. These are conventionally called the Bessel functions of the first and second kind.

\FloatBarrier
\subsubsection{Bessel Functions of the First and Second Kind}
\label{sec:bessel-first-second}

The Bessel functions of the first and second kind, $J_m$ and $Y_m$, respectively, form the standard pair of solutions. We use $J_m$ because it is regular at the origin for non-negative integer orders, whereas $Y_m$ is singular there.

The first-kind Bessel function has the series representation

\begin{equation}
\label{eq:bessel-series}
J_m(x)=\sum_{n=0}^{\infty}\frac{(-1)^n}{n!\Gamma(n+m+1)}\left(\frac{x}{2}\right)^{2n+m}.
\end{equation}

Here $\Gamma$ denotes the Gamma function.

For sufficiently large $x$, direct evaluation of the series can become numerically inefficient. We therefore also use the asymptotic form
\begin{equation}
\label{eq:bessel_approx}
J_m(x)=\sqrt{\frac{2}{\pi x}}\left[\cos\left(x-\frac{m\pi}{2}-\frac{\pi}{4}\right)+\mathcal{O}(|x|^{-1})\right].
\end{equation}

For integer $m$ we will also use
\begin{equation}
J_{-m}(x)=(-1)^mJ_m(x).
\end{equation}

\begin{figure}[h]
    \centering
    \includegraphics[width=0.55\linewidth]{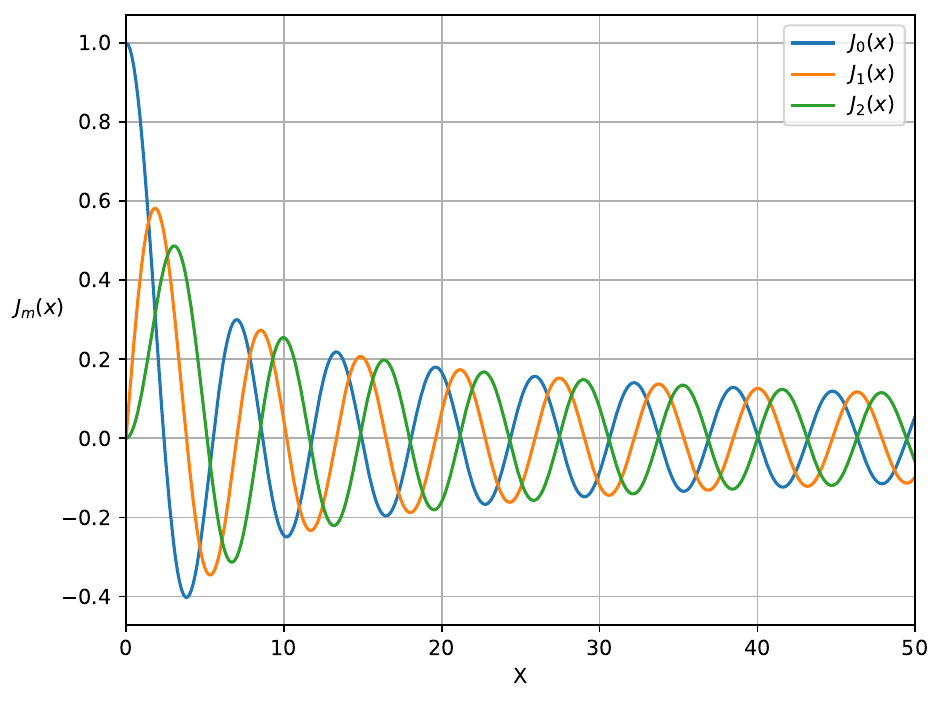}
    \caption{First-kind Bessel function $J_m$ evaluated at orders
    $m=0,1,2$. The numerical implementation switches to the asymptotic
    approximation in \cref{eq:bessel_approx} for sufficiently large $x$.}
    \label{fig:bessel1}
\end{figure}

\FloatBarrier
\subsubsection{Modified Bessel Functions}

Modified Bessel functions arise naturally in the Weber exponential integrals used later in the construction. The modified Bessel function of the first kind is
\begin{equation}
\label{eq:mod1}
I_m(x)=i^{-m}J_m(ix)=\sum_{n=0}^{\infty}\frac{1}{n!\Gamma(n+m+1)}\left(\frac{x}{2}\right)^{2n+m}.
\end{equation}

\begin{remark}
    For higher orders, when combined with an exponential factor, $I_m$ can grow large enough to cause numerical overflow. To avoid this we work with the exponentially scaled form. 

    \begin{equation}
        \widehat{I}_m(x) = I_m(x) \cdot e^{-|x|}
    \end{equation}
    \begin{equation}
         I_m(x) = e^{x} \cdot \widehat{I}_m(x)
    \end{equation}
\end{remark}

\begin{figure}[h]
    \centering
    \includegraphics[width=0.55\linewidth]{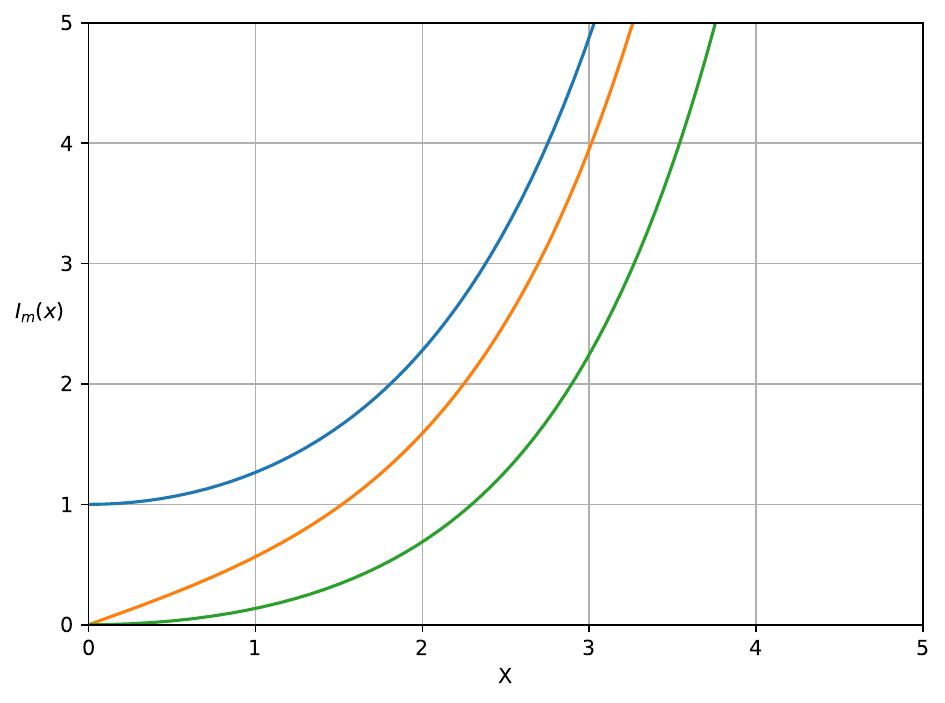}
    \caption{Modified first-kind Bessel function $I_m$ evaluated at
    orders $m=0,1,2$.}
    \label{fig:modbessel1}
\end{figure}

\FloatBarrier
\subsubsection{Derivative Identities for Bessel Functions}
\label{sec:firstderiv}

The derivative of the first-kind Bessel function is required when imposing the Neumann boundary condition.

\begin{lemma}[Derivative identity for $J_m$]
\label{lem:bessel-derivative}
For integer $m$,
\begin{equation}
J_m'(x)=\frac{1}{2}\left[J_{m-1}(x)-J_{m+1}(x)\right].
\label{eq:bessel-derivative}
\end{equation}
\end{lemma}

\begin{proof}
Starting from the series representation of the Bessel function of the first kind,
\begin{equation}
J_m(x)
=
\sum_{n=0}^{\infty}
\frac{(-1)^n}{n!\,\Gamma(n+m+1)}
\left(\frac{x}{2}\right)^{2n+m},
\end{equation}
we differentiate term by term. Since only
$\left(\frac{x}{2}\right)^{2n+m}$ depends on $x$, the chain rule gives
\begin{equation}
\frac{d}{dx}
\left[
\left(\frac{x}{2}\right)^{2n+m}
\right]
=
(2n+m)
\left(\frac{x}{2}\right)^{2n+m-1}
\frac{1}{2}.
\end{equation}
Therefore,
\begin{equation}
J'_m(x)
=
\sum_{n=0}^{\infty}
\frac{(-1)^n(2n+m)}
{2\,n!\,\Gamma(n+m+1)}
\left(\frac{x}{2}\right)^{2n+m-1}.
\label{eq:bessel_derivative_series}
\end{equation}

We now derive two equivalent expressions for $J'_m(x)$. The first is obtained by
splitting the factor $(2n+m)$ as
\[
2n+m = 2(n+m)-m.
\]
Substituting this into \eqref{eq:bessel_derivative_series} gives
\begin{equation}
J'_m(x)
=
\sum_{n=0}^{\infty}
\frac{(-1)^n[2(n+m)-m]}
{2\,n!\,\Gamma(n+m+1)}
\left(\frac{x}{2}\right)^{2n+m-1}.
\end{equation}

We distribute the sum into two separate series, $S_A$ and $S_B$:
\begin{equation}
J'_m(x)=S_A(x)+S_B(x),
\end{equation}
where
\begin{align}
S_A(x)
&=
\sum_{n=0}^{\infty}
\frac{(-1)^n\,2(n+m)}
{2\,n!\,\Gamma(n+m+1)}
\left(\frac{x}{2}\right)^{2n+m-1},
\\
S_B(x)
&=
\sum_{n=0}^{\infty}
\frac{(-1)^n(-m)}
{2\,n!\,\Gamma(n+m+1)}
\left(\frac{x}{2}\right)^{2n+m-1}.
\end{align}

For $S_A$, the factor of $2$ cancels:
\begin{equation}
S_A(x)
=
\sum_{n=0}^{\infty}
\frac{(-1)^n(n+m)}
{n!\,\Gamma(n+m+1)}
\left(\frac{x}{2}\right)^{2n+m-1}.
\end{equation}
Using the Gamma-function identity
\[
\Gamma(z+1)=z\,\Gamma(z),
\]
with $z=n+m$, we have
\[
\Gamma(n+m+1)=(n+m)\,\Gamma(n+m).
\]
Hence,
\begin{equation}
S_A(x)
=
\sum_{n=0}^{\infty}
\frac{(-1)^n}
{n!\,\Gamma(n+m)}
\left(\frac{x}{2}\right)^{2n+m-1}.
\end{equation}
Comparing this expression with the series definition of $J_{m-1}(x)$,
\[
J_{m-1}(x)
=
\sum_{n=0}^{\infty}
\frac{(-1)^n}
{n!\,\Gamma(n+(m-1)+1)}
\left(\frac{x}{2}\right)^{2n+(m-1)},
\]
we see that
\begin{equation}
S_A(x)=J_{m-1}(x).
\end{equation}

We now evaluate $S_B$. First, we rewrite the power of $x/2$ so that it has
the same exponent as the series representation of $J_m(x)$:
\begin{equation}
\left(\frac{x}{2}\right)^{2n+m-1}
=
\left(\frac{x}{2}\right)^{2n+m}
\left(\frac{x}{2}\right)^{-1}
=
\left(\frac{x}{2}\right)^{2n+m}
\frac{2}{x}.
\end{equation}
Therefore,
\begin{align}
S_B(x)
&=
\sum_{n=0}^{\infty}
\frac{(-1)^n(-m)}
{2\,n!\,\Gamma(n+m+1)}
\left(\frac{x}{2}\right)^{2n+m}
\frac{2}{x}
\\
&=
-\frac{m}{x}
\sum_{n=0}^{\infty}
\frac{(-1)^n}
{n!\,\Gamma(n+m+1)}
\left(\frac{x}{2}\right)^{2n+m}
\\
&=
-\frac{m}{x}J_m(x).
\end{align}
Consequently, the first derivative identity is
\begin{equation}
J'_m(x)
=
J_{m-1}(x)-\frac{m}{x}J_m(x).
\label{eq:bessel_derivative_identity_1}
\end{equation}

We now derive a second expression for $J'_m(x)$. Returning to
\eqref{eq:bessel_derivative_series}, we instead retain $(2n+m)$ in its
original form and split it as
\[
2n+m=m+2n.
\]
This gives two new series, $S_C$ and $S_D$, such that
\begin{equation}
J'_m(x)=S_C(x)+S_D(x),
\end{equation}
where
\begin{align}
S_C(x)
&=
\sum_{n=0}^{\infty}
\frac{(-1)^n m}
{2\,n!\,\Gamma(n+m+1)}
\left(\frac{x}{2}\right)^{2n+m-1},
\\
S_D(x)
&=
\sum_{n=0}^{\infty}
\frac{(-1)^n\,2n}
{2\,n!\,\Gamma(n+m+1)}
\left(\frac{x}{2}\right)^{2n+m-1}.
\end{align}

The first of these sums is the negative of $S_B$, so
\begin{equation}
S_C(x)
=
\frac{m}{x}J_m(x).
\end{equation}

For the second series, the factor of $2$ cancels. Furthermore, the $n=0$ term vanishes because of the factor $n$. We may therefore begin the sum at $n=1$:
\begin{equation}
S_D(x)
=
\sum_{n=1}^{\infty}
\frac{(-1)^n n}
{n!\,\Gamma(n+m+1)}
\left(\frac{x}{2}\right)^{2n+m-1}.
\end{equation}
Using
\[
n!=n(n-1)!,
\]
the factor of $n$ cancels, giving
\begin{equation}
S_D(x)
=
\sum_{n=1}^{\infty}
\frac{(-1)^n}
{(n-1)!\,\Gamma(n+m+1)}
\left(\frac{x}{2}\right)^{2n+m-1}.
\end{equation}

To express this in the standard series form for a Bessel function, we
now shift the summation index. Let
\[
k=n-1,
\qquad\text{so that}\qquad
n=k+1.
\]
Since $n$ begins at $1$, the new index $k$ begins at $0$. We transform
each component of the summand individually.

First, the alternating sign becomes
\begin{equation}
(-1)^n
=
(-1)^{k+1}
=
(-1)^k(-1)
=
-(-1)^k.
\end{equation}

Second, the factorial becomes
\begin{equation}
(n-1)!
=
((k+1)-1)!
=
k!.
\end{equation}

Third, the Gamma-function argument transforms as
\begin{align}
\Gamma(n+m+1)
&=
\Gamma((k+1)+m+1)
\\
&=
\Gamma(k+m+2)
\\
&=
\Gamma\bigl(k+(m+1)+1\bigr).
\end{align}

Finally, the exponent of $x/2$ becomes
\begin{align}
2n+m-1
&=
2(k+1)+m-1
\\
&=
2k+m+1
\\
&=
2k+(m+1).
\end{align}

Substituting all of these transformations into $S_D(x)$ gives
\begin{align}
S_D(x)
&=
-\sum_{k=0}^{\infty}
\frac{(-1)^k}
{k!\,\Gamma\bigl(k+(m+1)+1\bigr)}
\left(\frac{x}{2}\right)^{2k+(m+1)}
\\
&=
-J_{m+1}(x).
\end{align}
Therefore,
\begin{equation}
J'_m(x)
=
\frac{m}{x}J_m(x)-J_{m+1}(x).
\label{eq:bessel_derivative_identity_2}
\end{equation}

Combining the two identities,
\eqref{eq:bessel_derivative_identity_1} and \eqref{eq:bessel_derivative_identity_2},
allows the terms involving $mJ_m(x)/x$ to cancel:
\begin{align}
2J'_m(x)
&=
\left(
J_{m-1}(x)-\frac{m}{x}J_m(x)
\right)
+
\left(
\frac{m}{x}J_m(x)-J_{m+1}(x)
\right)
\\
&=
J_{m-1}(x)-J_{m+1}(x).
\end{align}
Dividing by $2$ yields the desired derivative identity,
\begin{equation}
J'_m(x)
=
\frac{1}{2}
\left[
J_{m-1}(x)-J_{m+1}(x)
\right].
\end{equation}
\end{proof}

\begin{figure}[h]
    \centering
    \includegraphics[width=0.55\linewidth]{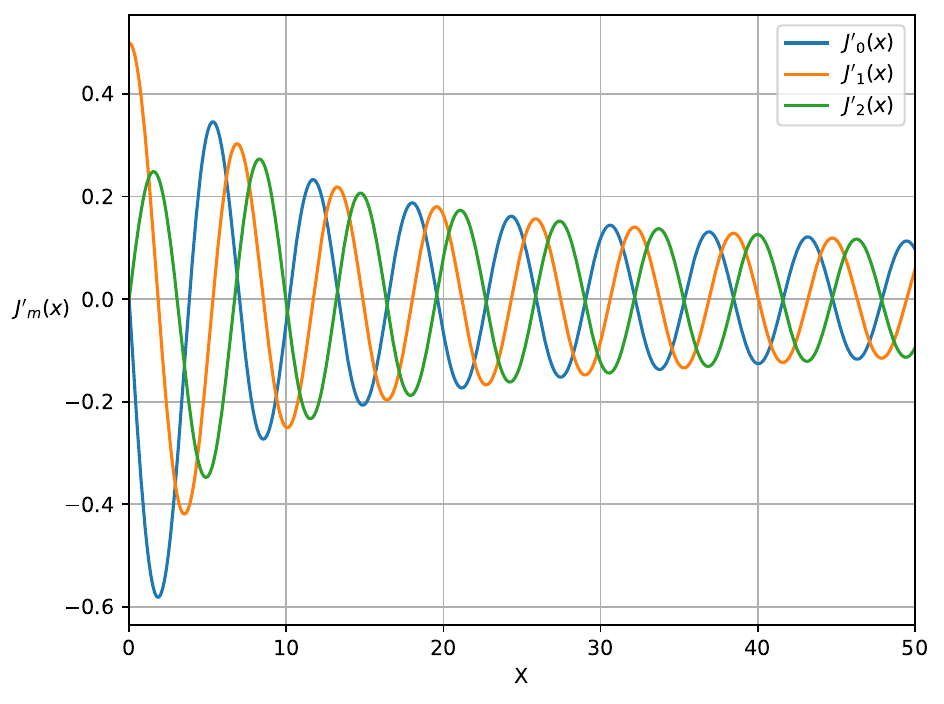}
    \caption{First-kind Bessel derivative $J'_m$ evaluated at orders
    $m=0,1,2$.}
    \label{fig:derivbessel1}
\end{figure}

\FloatBarrier
\subsection{Wavelet Properties}
\label{sec:properties}

Before constructing the Fourier-Bessel wavelets, we introduce the properties used throughout the remainder of the notes.

\begin{definition}[Zero-mean condition]
\label{def:zero-mean}
For the wavelet family considered here, we impose the zero-mean condition
\begin{equation}
\int_{\R^d}\psi(x)\,dx=0.
\label{eq:zero-mean}
\end{equation}
This ensures that the wavelet has no response to a spatially constant
component.
\end{definition}

We additionally impose an $L^p$ normalisation according to the intended application. An $L^2$ normalisation is appropriate when the total energy of the wavelet should remain constant. For frequency-domain peak consistency, we instead use a normalisation based on the radial Fourier
response. The $L^1$ norm provides the useful bound established below.

\begin{theorem}[Plancherel's theorem]
\label{thm:plancherel}
Under the one-dimensional Fourier-transform convention
\begin{equation}
\widehat{f}(\omega)=\int_{\R}f(t)e^{-i\omega t}\,dt,
\label{eq:fourier-convention-1d}
\end{equation}
we have
\begin{equation}
\fnorm{f}_2^2=\frac{1}{2\pi}\fnorm{\widehat{f}}_2^2.
\label{eq:plancherel}
\end{equation}
\end{theorem}

\begin{proof}
Let $f, g \in L^1(\mathbb{R}) \cap L^2(\mathbb{R})$. Their inner product is
\begin{equation}
    \label{eq:inner}
   \langle f, g \rangle = \int_{\mathbb{R}} f(t)\overline{g(t)}\, dt.  
\end{equation}

Using the inverse Fourier transform,
\begin{equation}
    f(t) = \frac{1}{2\pi}\int_{\mathbb{R}} \hat f(\omega) e^{i\omega t}\, d\omega, 
\end{equation}

we substitute into \cref{eq:inner} and, invoking Fubini's theorem to exchange the order of integration
(justified since $f,g \in L^1 \cap L^2$), obtain

\begin{equation}
\label{eq:fubini}
    \langle f, g \rangle
    = \frac{1}{2\pi}\int_{\mathbb{R}} \hat f(\omega)
    \left( \int_{\mathbb{R}} \overline{g(t)}\, e^{i\omega t}\, dt \right) d\omega.   
\end{equation}

We now identify the bracketed integral. By definition \cref{eq:fourier-convention-1d}, the Fourier transform of $g$ is
$\hat g(\omega) = \int_{\mathbb{R}} g(t) e^{-i\omega t}\, dt$, so its complex conjugate is

\begin{equation}
\label{eq:conjugate}
    \overline{\hat g(\omega)}
    = \overline{\int_{\mathbb{R}} g(t) e^{-i\omega t}\, dt}
    = \int_{\mathbb{R}} \overline{g(t)}\, e^{i\omega t}\, dt,
\end{equation}

which is exactly the bracketed term in \cref{eq:fubini}. Substituting this identification gives

\begin{equation}
  \langle f, g \rangle = \frac{1}{2\pi} \int_{\mathbb{R}} \hat f(\omega) \overline{\hat g(\omega)}\, d\omega.
\end{equation}

Setting $g = f$ gives $\langle f,f\rangle = \|f\|_2^2$ on the left and
$\frac{1}{2\pi}\|\hat f\|_2^2$ on the right, which is \cref{eq:plancherel}.
\end{proof}

\begin{proposition}[$L^1$-$L^\infty$ Fourier bound]
\label{prop:l1-linf}
For $f\in L^1(\R^d)$,
\begin{equation}
\fnorm{\widehat{f}}_\infty\leq\fnorm{f}_1.
\label{eq:l1-linf}
\end{equation}
\end{proposition}

\begin{proof}
For every frequency $\omega$,
\begin{align}
\fabs{\widehat{f}(\omega)}
&=\fabs{\int_{\R^d}
f(x)e^{-i\omega\cdot x}\,dx} \\
&\leq\int_{\R^d}\fabs{f(x)}\fabs{e^{-i\omega\cdot x}}\,dx \\
&=\int_{\R^d}\fabs{f(x)}\,dx=\fnorm{f}_1.
\end{align}
Taking the supremum over $\omega$ proves \cref{eq:l1-linf}.
\end{proof}

\FloatBarrier
\section{Fourier-Bessel Disk Harmonics}
\label{sec:disk-harmonics}

We now construct the Fourier-Bessel basis on the planar unit disk
\begin{equation}
\D=\left\{(\rho,\varphi):0\leq\rho\leq1,\;0\leq\varphi<2\pi\right\}.
\end{equation}

Fourier-Bessel basis functions arise as eigenfunctions of the Laplacian on the disk. Equivalently, they solve the Helmholtz equation subject to a boundary condition at $\rho=1$.

\FloatBarrier
\subsection{The Helmholtz Equation}

We consider
\begin{equation}
\label{eq:helm}
\frac{\partial^2f}{\partial\rho^2}
+
\frac{1}{\rho}\frac{\partial f}{\partial\rho}
+
\frac{1}{\rho^2}\frac{\partial^2f}{\partial\varphi^2}
=
-\lambda^2 f.
\end{equation}

We seek separable solutions of the form

\begin{equation}
  f(\rho,\varphi) = R(\rho)\Phi(\varphi).
\end{equation}

Substituting into \cref{eq:helm} gives

\begin{equation}
  \Phi \frac{d^2 R}{d\rho^2} + \frac{\Phi}{\rho}\frac{dR}{d\rho}
+ \frac{R}{\rho^2}\frac{d^2\Phi}{d\varphi^2} = -\lambda^2 R \Phi. 
\end{equation}

Dividing both sides by $R\Phi$ isolates the radial and angular dependence:

\begin{equation}
\frac{1}{R}\frac{d^2 R}{d\rho^2} + \frac{1}{\rho R}\frac{dR}{d\rho}
+ \frac{1}{\rho^2 \Phi}\frac{d^2\Phi}{d\varphi^2} = -\lambda^2.   
\end{equation}

Multiplying through by $\rho^2$ and collecting the angular term on one side gives

\begin{equation}
 \rho^2 \frac{R''}{R} + \rho \frac{R'}{R} + \lambda^2\rho^2
= -\frac{\Phi''}{\Phi}. 
\end{equation}

The left-hand side depends only on $\rho$ and the right-hand side only on $\varphi$. Both must therefore equal a common constant, which we write as $m^2$:

\begin{equation}
 -\frac{\Phi''}{\Phi} = \frac{\rho^2 R'' + \rho R' + \lambda^2\rho^2 R}{R} = m^2.  
\end{equation}

\begin{remark}
    We fix the sign of the separation constant as $+m^2$ rather than $-m^2$.
    This is required for the angular equation to admit periodic (rather than exponentially
    growing/decaying) solutions, since $2\pi$-periodicity of $\Phi$ is a physical requirement
    on the unit disk. The choice of sign is verified immediately below.
\end{remark} 

The angular equation is therefore

\begin{equation}
  \Phi'' + m^2 \Phi = 0.
\end{equation}

Imposing $2\pi$-periodicity gives

\begin{equation}
  \Phi_m(\varphi) = e^{im\varphi}, \qquad m \in \mathbb{Z}. 
\end{equation}

The radial equation is
\begin{equation}
\rho^2R''+\rho R'
+
(\lambda^2\rho^2-m^2)R
=
0,
\label{eq:radial-bessel}
\end{equation}
which is Bessel's equation under the change of variable
$x=\lambda\rho$. The solution regular at the origin is therefore
\begin{equation}
R_{m,k}(\rho)
=
J_m(\lambda_{m,k}\rho).
\end{equation}

\FloatBarrier
\subsection{Neumann Boundary Condition}

To obtain the disk harmonics used here, we impose the Neumann condition
\begin{equation}
\left.
\frac{\partial R_{m,k}}{\partial\rho}
\right|_{\rho=1}
=
0.
\end{equation}

Since
\begin{equation}
\frac{\partial}{\partial\rho}
J_m(\lambda_{m,k}\rho)
=
\lambda_{m,k}J_m'(\lambda_{m,k}\rho),
\end{equation}
the boundary condition is equivalent to
\begin{equation}
J_m'(\lambda_{m,k})=0.
\label{eq:neumann-condition}
\end{equation}

\begin{proposition}[Neumann eigenvalues]
\label{prop:neumann-eigenvalues}
For each angular order $m$, the radial eigenvalues are given by the positive roots $\lambda_{m,k}$ of
\cref{eq:neumann-condition}.
\end{proposition}

\begin{definition}[Fourier-Bessel disk harmonic]
\label{def:disk-harmonic}
For angular order $m$ and root index $k$, define
\begin{equation}
D_{m,k}(\rho,\varphi)
=
N_{m,k}
J_m(\lambda_{m,k}\rho)e^{im\varphi},
\label{eq:disk-bessel}
\end{equation}
where $N_{m,k}$ is chosen according to the desired basis normalisation. For the orthonormal disk basis used in \cite{Shaqfa_2025}, the normalisation is
\begin{equation}
\label{eq:norm_ortho}
    N_{m,k}=\frac{J_m(\lambda_{m,k})^{-1}}{\sqrt{\pi(1-\frac{m^2}{\lambda_{m,k}^2}})}
\end{equation}
\end{definition}

\begin{remark}
The distinction between the root index $k$ and the eigenvalue
$\lambda_{m,k}$ is important throughout the construction. In particular, $k$ is a discrete index, whereas $\lambda_{m,k}$ determines the radial oscillation frequency.
\end{remark}

\FloatBarrier
\subsection{Root Finding Using Muller's Method}
\label{sec:muller}

The eigenvalues are obtained by locating the roots of $J_m'$. Following \cite{Shaqfa_2025}, we use Muller's method. Given three starting estimates, the method iteratively fits a quadratic parabola and uses one of its roots as the next approximation. It can therefore be viewed as a higher-order extension of the secant method.

For initial estimates, we use McMahon's asymptotic expansion \cite{Abramowitz1964,DLMF}. Keeping the first correction term gives
\begin{equation}
\lambda_{m,k}\approx\beta_{m,k}-\frac{4m^2+3}{8\beta_{m,k}},
\label{eq:mcmahon}
\end{equation}
where
\begin{equation}
\beta_{m,k}=\left(k+\frac{m}{2}-\frac{3}{4}\right)\pi.
\label{eq:beta}
\end{equation}

The leading term in \cref{eq:mcmahon} is linear in the root index $k$ with slope $\pi$. This explains why the spacing between consecutive eigenvalues approaches $\pi$ asymptotically. The correction term also shows why larger angular orders require larger $k$ before this limiting
spacing becomes apparent.

\begin{remark}
The asymptotic expansion is used here primarily to initialise numerical root finding and to interpret the frequency spacing. The numerical eigenvalues used in the wavelet construction are obtained from the roots of $J_m'$ rather than from the asymptotic approximation alone.
\end{remark}

\begin{figure}[h]
    \centering
    \includegraphics[width=0.8\linewidth]{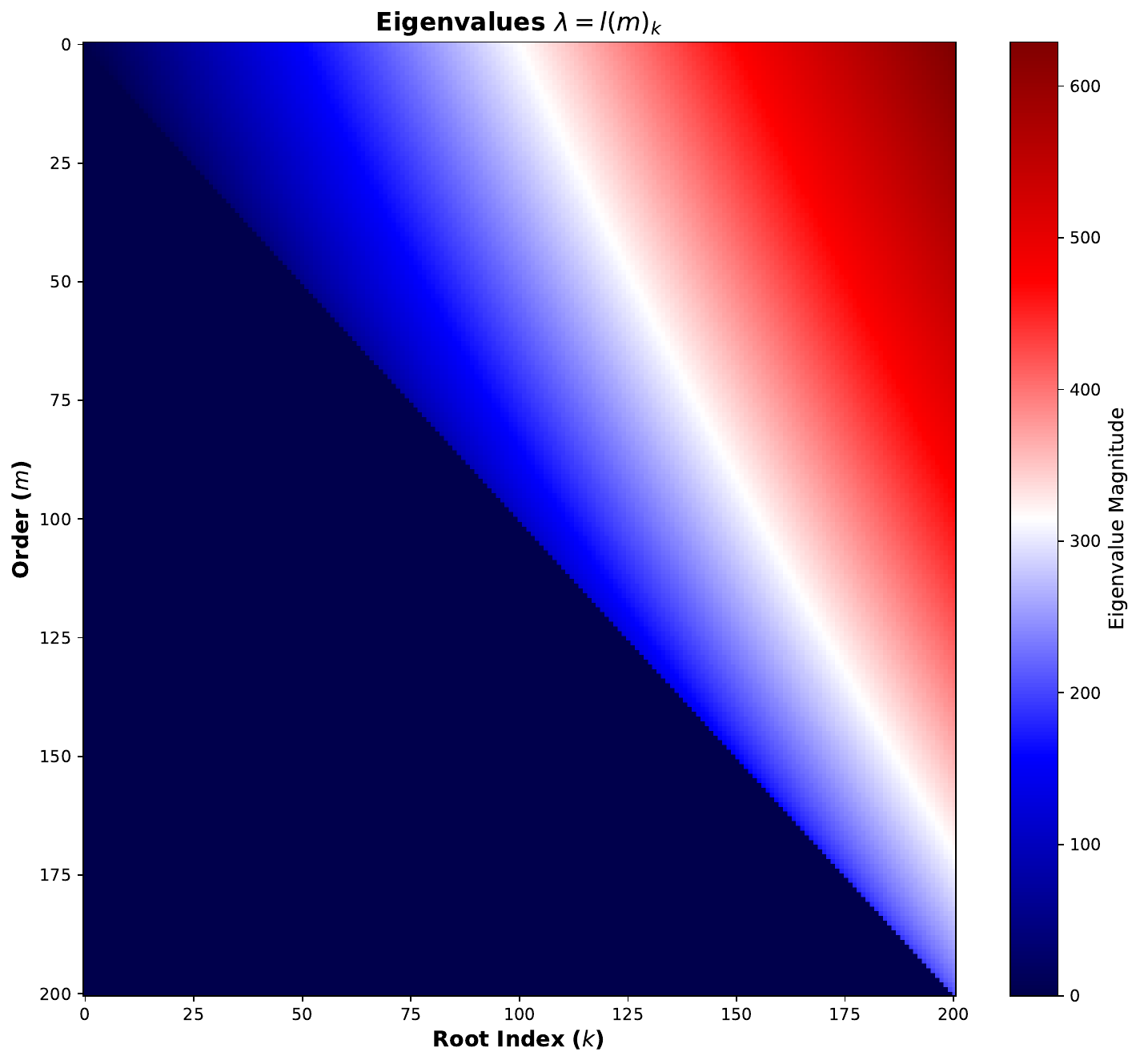}
    \caption{Reproduction of \cite{Shaqfa_2025}: eigenvalues associated
    with the roots of $J_m'$ for different angular orders and root
    indices.}
    \label{fig:eigens}
\end{figure}

\FloatBarrier
\section{Construction of Fourier-Bessel Wavelets}

Having established the Fourier-Bessel basis on the unit disk, we now extend it to a wavelet family defined on the continuous plane. We use a Gaussian envelope to localise the oscillatory basis function in space. For notational simplicity, within this section we sometimes write $\lambda=\lambda_{m,k}$ when the indices are unambiguous.

\begin{definition}[Fourier-Bessel wavelet]
\label{def:fb-wavelet}
The spatial Fourier-Bessel wavelet of angular order $m$ and root index
$k$ is
\begin{equation}
\psi_{m,k}(\rho,\varphi)
=
N_{m,k}
e^{-\rho^2/(2\sigma^2)}
\left[
J_m(\lambda_{m,k}\rho)-K_{m,k}
\right]
e^{im\varphi},
\label{eq:fb-wavelet}
\end{equation}
where $K_{m,k}$ enforces the zero-mean condition and $N_{m,k}$ denotes the application-specific normalisation.
\end{definition}

\begin{remark}[Radial and angular control]
Unlike an affine wavelet construction in which a mother wavelet is repeatedly scaled and rotated, the present family keeps the Gaussian envelope fixed while using $\lambda_{m,k}$ to control radial oscillation. Angular selectivity is provided directly by the factor $e^{im\varphi}$.
\end{remark}

\begin{figure}[h]
    \centering
    \includegraphics[width=0.8\linewidth]{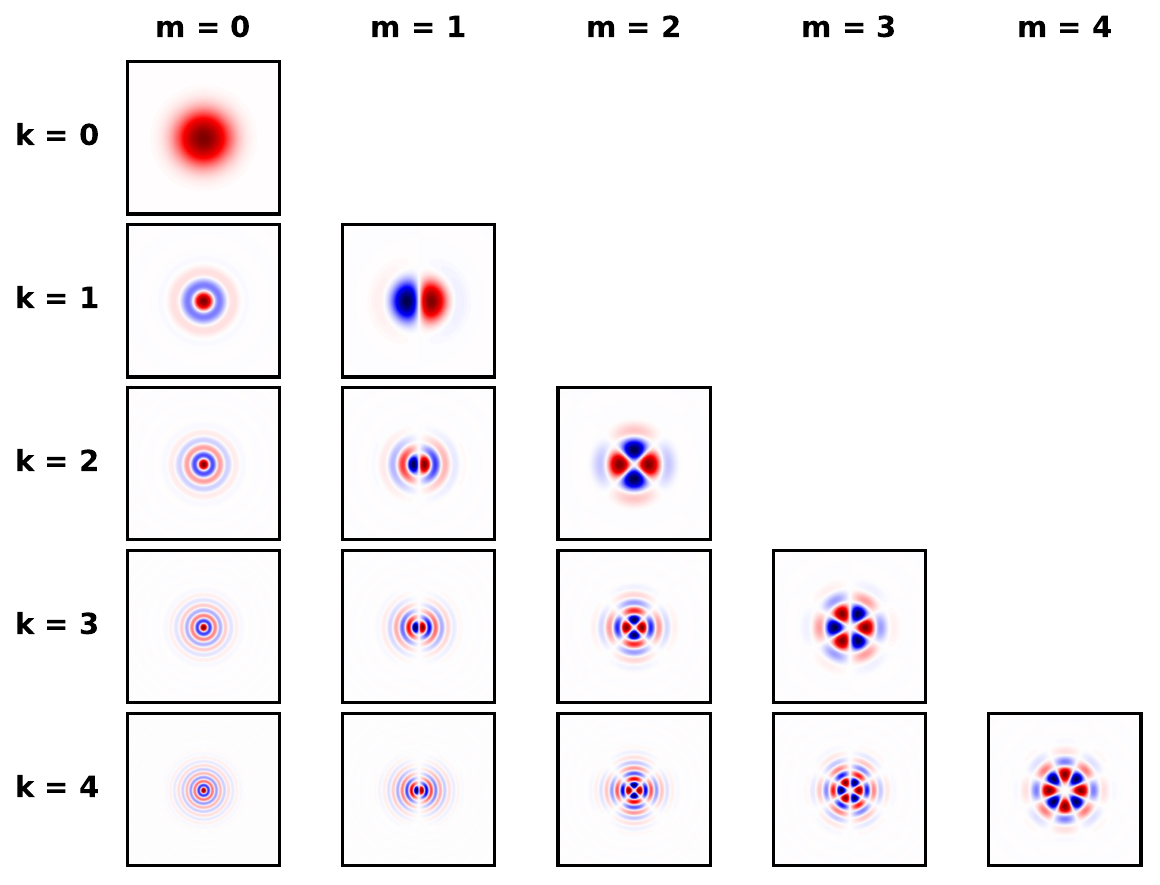}
    \caption{Fourier-Bessel wavelet bank for $m,k=5$. Only the real
    components are shown. The imaginary components are obtained by a
    $\pi/(2m)$ angular phase shift. Here $\sigma=1$.}
    \label{fig:spatial}
\end{figure}

\begin{figure}[h]
    \centering
    \includegraphics[width=1\linewidth]{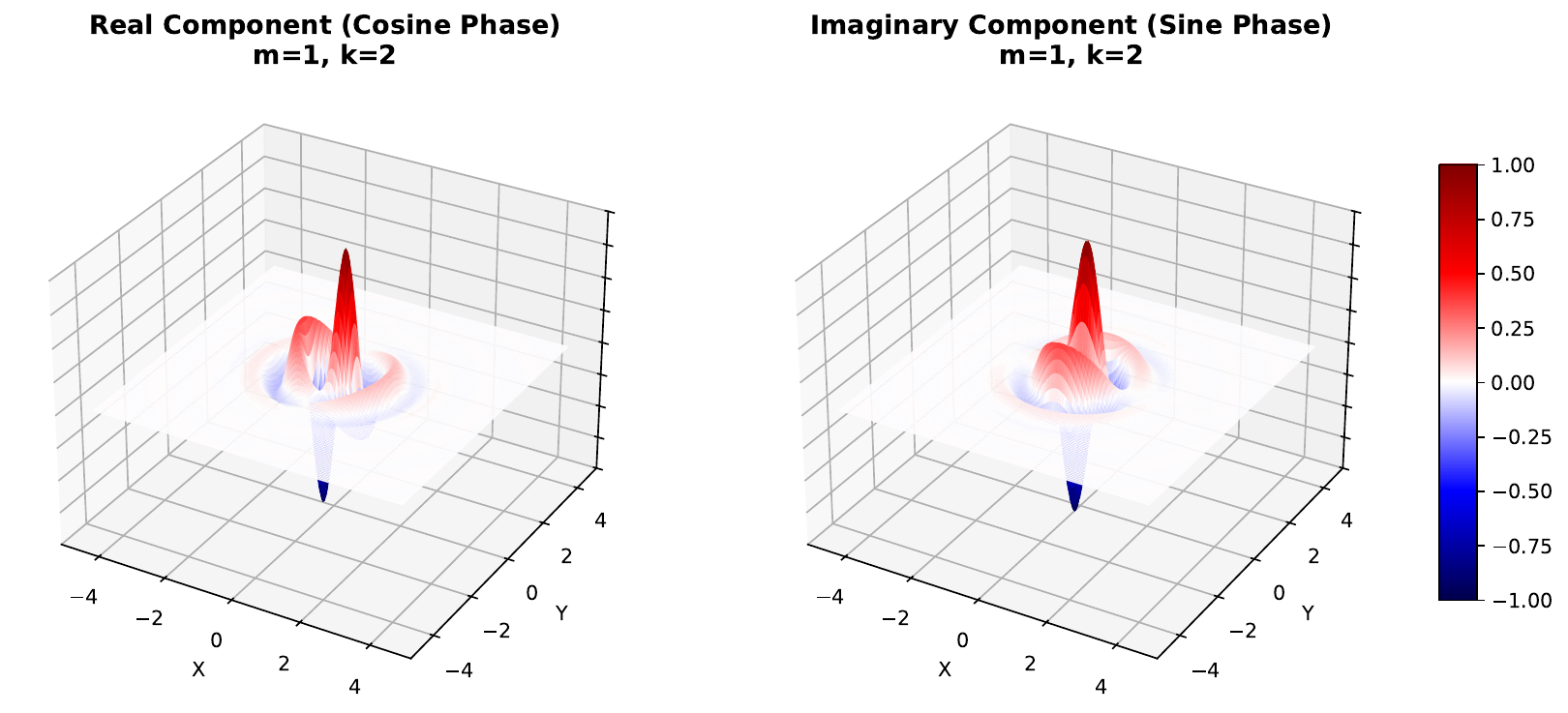}
    \caption{3D rendering of the real and imaginary parts of the spatial
    Fourier-Bessel wavelet with $m=1$ and $k=2$.}
    \label{fig:3d_spatial}
\end{figure}

\FloatBarrier
\subsection{Low-Pass Filter}

To cover the remaining low frequencies, we introduce a Gaussian low-pass filter,
\begin{equation}
\phi(\rho)=\frac{1}{2\pi\sigma^2}\exp\left(-\frac{\rho^2}{2\sigma^2}\right).
\label{eq:lowpass-spatial}
\end{equation}

Under the Fourier-transform convention used later, its frequency-domain representation is
\begin{equation}
\widehat{\phi}(q)=\exp\left(-\frac{\sigma^2q^2}{2}\right).
\label{eq:lowpass-fourier}
\end{equation}

\FloatBarrier
\subsection{Zero-Mean Correction}
\label{sec:zero-mean-correction}

For $m\geq1$, the angular factor satisfies
\begin{equation}
\int_0^{2\pi}e^{im\varphi}\,d\varphi=0,
\end{equation}
so the zero-mean condition is automatically satisfied. The case $m=0$ requires an explicit correction.

\begin{proposition}[Zero-mean correction for $m=0$]
\label{prop:zero-mean}
For $m=0$, the unique constant $K_{0,k}$ satisfying
\begin{equation}
\int_{\R^2}\psi_{0,k}(x)\,dx=0
\end{equation}
is
\begin{equation}
K_{0,k}=\exp\left(-\frac{\lambda_{0,k}^2\sigma^2}{2}\right).
\label{eq:zero-mean-K}
\end{equation}
\end{proposition}

\begin{proof}
The zero-mean condition reduces to
\begin{equation}
\int_0^\infty e^{-\rho^2/(2\sigma^2)}\left[J_0(\lambda\rho)-K\right]\rho\,d\rho=0.
\end{equation}

Separating the two terms gives
\begin{equation}
\int_0^\infty e^{-\rho^2/(2\sigma^2)}J_0(\lambda\rho)\rho\,d\rho=
K\int_0^\infty e^{-\rho^2/(2\sigma^2)}\rho\,d\rho.
\end{equation}

The first integral is Weber's first exponential integral
\cite{watson1995treatise},
\begin{equation}
\int_0^\infty e^{-t^2\rho^2}J_0(\lambda\rho)\rho\,d\rho
=\frac{1}{2t^2}\exp\left(-\frac{\lambda^2}{4t^2}\right).
\label{eq:weber-first}
\end{equation}

Setting $t^2=1/(2\sigma^2)$ gives
\begin{equation}
\int_0^\infty e^{-\rho^2/(2\sigma^2)}J_0(\lambda\rho)\rho\,d\rho
=\sigma^2e^{-\lambda^2\sigma^2/2}.
\end{equation}

For the second integral, the substitution
$u=\rho^2/(2\sigma^2)$ gives
\begin{equation}
\int_0^\infty e^{-\rho^2/(2\sigma^2)}\rho\,d\rho=\sigma^2.
\end{equation}

Therefore,
\begin{equation}
\sigma^2e^{-\lambda^2\sigma^2/2}-K\sigma^2=0,
\end{equation}
and hence \cref{eq:zero-mean-K}.
\end{proof}

\FloatBarrier
\subsection{$L^2$ Normalisation}

We now choose the normalisation constant so that
\begin{equation}
\fnorm{\psi_{m,k}}_2=1.
\end{equation}
The $m=0$ and $m\geq1$ cases differ because the zero-mean correction is present only for $m=0$.

\subsubsection{Angular orders $m\geq1$}

For $m\geq1$, define
\begin{equation}
N^{(2)}_{m,k}
=
\left[
\pi\sigma^2
e^{-\lambda_{m,k}^2\sigma^2/2}
I_m\left(
\frac{\lambda_{m,k}^2\sigma^2}{2}
\right)
\right]^{-1/2}.
\label{eq:l2-positive}
\end{equation}

\begin{proposition}[$L^2$ normalisation for $m\geq1$]
\label{prop:l2-positive}
The constant in \cref{eq:l2-positive} satisfies
$\fnorm{\psi_{m,k}}_2=1$ for $m\geq1$.
\end{proposition}

\begin{proof}
For $m\geq1$, $K_{m,k}=0$, so
\begin{align}
1
&=
2\pi\left(N^{(2)}_{m,k}\right)^2
\int_0^\infty
e^{-\rho^2/\sigma^2}
J_m(\lambda\rho)^2
\rho\,d\rho.
\end{align}

We use Weber's second exponential integral
\cite{watson1995treatise},
\begin{equation}
\int_0^\infty
e^{-t^2\rho^2}
J_m(\lambda_1\rho)
J_m(\lambda_2\rho)
\rho\,d\rho
=
\frac{1}{2t^2}
e^{-(\lambda_1^2+\lambda_2^2)/(4t^2)}
I_m\left(
\frac{\lambda_1\lambda_2}{2t^2}
\right).
\label{eq:weber-second}
\end{equation}

Setting $\lambda_1=\lambda_2=\lambda$ and $t^2=1/\sigma^2$ gives
\begin{equation}
\int_0^\infty
e^{-\rho^2/\sigma^2}
J_m(\lambda\rho)^2
\rho\,d\rho
=
\frac{\sigma^2}{2}
e^{-\lambda^2\sigma^2/2}
I_m\left(\frac{\lambda^2\sigma^2}{2}\right).
\end{equation}

Substitution into the norm condition yields
\begin{equation}
\left(N^{(2)}_{m,k}\right)^2
\pi\sigma^2
e^{-\lambda^2\sigma^2/2}
I_m\left(\frac{\lambda^2\sigma^2}{2}\right)
=
1,
\end{equation}
which gives \cref{eq:l2-positive}.
\end{proof}

\subsubsection{Zeroth angular order $m=0$}

For $m=0$, the zero-mean correction contributes to the norm. Define
\begin{equation}
N^{(2)}_{0,k}
=
\left[
\pi\sigma^2
\left(
e^{-\lambda^2\sigma^2/2}
I_0\left(\frac{\lambda^2\sigma^2}{2}\right)
-
2e^{-3\sigma^2\lambda^2/4}
+
e^{-\sigma^2\lambda^2}
\right)
\right]^{-1/2}.
\label{eq:l2-zero}
\end{equation}

\begin{proposition}[$L^2$ normalisation for $m=0$]
\label{prop:l2-zero}
The constant in \cref{eq:l2-zero} satisfies
$\fnorm{\psi_{0,k}}_2=1$.
\end{proposition}

\begin{proof}
Expanding the squared correction gives
\begin{align}
1
&=
2\pi\left(N^{(2)}_{0,k}\right)^2
\int_0^\infty
e^{-\rho^2/\sigma^2}
\left[J_0(\lambda\rho)-K\right]^2
\rho\,d\rho \\
&=
2\pi\left(N^{(2)}_{0,k}\right)^2
\left(I_A+I_B+I_C\right),
\end{align}
where
\begin{align}
I_A
&=
\int_0^\infty
e^{-\rho^2/\sigma^2}
J_0(\lambda\rho)^2\rho\,d\rho,\\
I_B
&=
-2K
\int_0^\infty
e^{-\rho^2/\sigma^2}
J_0(\lambda\rho)\rho\,d\rho,\\
I_C
&=
K^2
\int_0^\infty
e^{-\rho^2/\sigma^2}\rho\,d\rho.
\end{align}

By \cref{eq:weber-second},
\begin{equation}
I_A
=
\frac{\sigma^2}{2}
e^{-\lambda^2\sigma^2/2}
I_0\left(\frac{\lambda^2\sigma^2}{2}\right).
\end{equation}

Using \cref{eq:weber-first} with $t^2=1/\sigma^2$,
\begin{equation}
I_B
=
-K\sigma^2e^{-\lambda^2\sigma^2/4}
=
-\sigma^2e^{-3\sigma^2\lambda^2/4},
\end{equation}
where we used \cref{eq:zero-mean-K}. Finally,
\begin{equation}
I_C
=
\frac{K^2\sigma^2}{2}
=
\frac{\sigma^2}{2}e^{-\sigma^2\lambda^2}.
\end{equation}

Combining the three terms gives \cref{eq:l2-zero}.
\end{proof}

\subsection{Fourier-Domain Peak Normalisation}

The $L^1$-$L^\infty$ bound in \cref{prop:l1-linf} motivates an
$L^1$-based normalisation when frequency-domain amplitude consistency is desired. In the implementation, however, we directly normalise each wavelet by the maximum of its radial Fourier response.

\begin{definition}[Peak normalisation]
\label{def:peak-normalisation}
Let $R_{m,k}(q)$ denote the radial component of the Fourier-domain wavelet. We define
\begin{equation}
N^{(1)}_{m,k}
=
\frac{1}
{\displaystyle\max_{q\geq0}\fabs{R_{m,k}(q)}}.
\label{eq:peak-normalisation}
\end{equation}
\end{definition}

The maximum can be found numerically on the finite frequency grid used in the implementation. We nevertheless derive the corresponding stationary-point equations below.

The modified Bessel function satisfies
\begin{align}
I_m'(x)
&=
I_{m-1}(x)-\frac{m}{x}I_m(x)\\
&=
\frac{m}{x}I_m(x)+I_{m+1}(x),
\end{align}
and hence
\begin{equation}
I_m'(x)
=
\frac12\left[I_{m-1}(x)+I_{m+1}(x)\right].
\label{eq:modified-bessel-derivative}
\end{equation}

\subsubsection{Angular orders $m\geq1$}

For $m\geq1$, set
\begin{equation}
x=\sigma^2\lambda q.
\end{equation}
The radial response is
\begin{equation}
R_{m,k}(q)
=
\sigma^2
e^{-\sigma^2(\lambda^2+q^2)/2}
I_m(x).
\label{eq:radial-response-positive}
\end{equation}

Differentiating and setting the derivative to zero gives
\begin{align}
0
&=
\frac{dR_{m,k}}{dq}\\
&=
\sigma^2
e^{-\sigma^2(\lambda^2+q^2)/2}
\left[
-\sigma^2qI_m(x)
+
\sigma^2\lambda I_m'(x)
\right].
\end{align}

Since the prefactors are non-zero,
\begin{equation}
\lambda I_m'(x)
=
qI_m(x).
\label{eq:peak-condition-positive}
\end{equation}

Using the first derivative identity for $I_m$ gives
\begin{equation}
\lambda I_{m-1}(x)
=
I_m(x)
\left(
q+\frac{m}{\sigma^2q}
\right).
\label{eq:qstar-positive}
\end{equation}

In general, \cref{eq:qstar-positive} does not admit a closed-form
solution for $q$. The peak frequency $q_{m,k}^\ast$ is therefore found
numerically.

\subsubsection{Zeroth angular order $m=0$}

For $m=0$, the correction term gives
\begin{equation}
R_{0,k}(q)
=
\sigma^2
e^{-\sigma^2(\lambda^2+q^2)/2}
\left[I_0(x)-1\right].
\label{eq:radial-response-zero}
\end{equation}

Differentiating gives the stationary-point equation
\begin{equation}
\lambda I_0'(x)
=
q\left[I_0(x)-1\right].
\end{equation}

Since $I_0'(x)=I_1(x)$,
\begin{equation}
\lambda I_1(x)
=
q\left[I_0(x)-1\right].
\label{eq:qstar-zero}
\end{equation}

Again, the peak frequency $q_{0,k}^\ast$ is obtained numerically.

\subsubsection{Asymptotic case}

When $\sigma^2\lambda^2\gg1$, the large-argument asymptotic behaviour of
the modified Bessel function gives
\begin{equation}
I_m'(x)\approx I_m(x).
\end{equation}

Consequently, \cref{eq:peak-condition-positive} gives
\begin{equation}
q_{m,k}^\ast\approx\lambda_{m,k}.
\end{equation}

Evaluating the radial response at this approximate maximum yields
\begin{align}
N^{(1)}_{m,k}
&\approx
\frac{1}
{\sigma^2e^{-\sigma^2\lambda_{m,k}^2}
I_m(\sigma^2\lambda_{m,k}^2)},
\qquad m\geq1,\\
N^{(1)}_{0,k}
&\approx
\frac{1}
{\sigma^2e^{-\sigma^2\lambda_{0,k}^2}
\left[
I_0(\sigma^2\lambda_{0,k}^2)-1
\right]}.
\end{align}

\FloatBarrier
\section{Fourier-Domain Representation}

The previous sections constructed the wavelets in the spatial domain. We now derive their closed-form Fourier representation. This form is useful both for analysing frequency coverage and for implementing the
filters without explicitly computing a numerical Fourier transform.

We use the two-dimensional Fourier-transform convention
\begin{equation}
\widehat{f}(k_x,k_y)=\int_{\R^2}f(x,y)e^{-i(k_xx+k_yy)}\,dx\,dy.
\label{eq:fourier-convention-2d}
\end{equation}

Writing the spatial and frequency coordinates in polar form,
\begin{equation}
x=\rho\cos\varphi,\qquad y=\rho\sin\varphi,
\end{equation}
and
\begin{equation}
k_x=q\cos\phi,\qquad k_y=q\sin\phi,
\end{equation}
we have
\begin{equation}
k_xx+k_yy=q\rho\cos(\varphi-\phi).
\end{equation}

\begin{lemma}[Angular Fourier-Bessel integral]
\label{lem:angular-integral}
For integer $m$,
\begin{equation}
\int_0^{2\pi}e^{im\varphi}e^{-iq\rho\cos(\varphi-\phi)}\,d\varphi=2\pi (-i)^m e^{im\phi}J_m(q\rho).
\label{eq:angular-integral}
\end{equation}
\end{lemma}

\begin{proof}
Set $\theta=\varphi-\phi$. Then
\begin{align}
\int_0^{2\pi}e^{im\varphi}e^{-iq\rho\cos(\varphi-\phi)}d\varphi&=e^{im\phi}
\int_0^{2\pi}e^{im\theta}e^{-iq\rho\cos\theta}\,d\theta.
\end{align}

Using the Jacobi-Anger expansion
\begin{equation}
e^{-iz\cos\theta}=\sum_{n=-\infty}^{\infty}
(-i)^nJ_n(z)e^{in\theta},
\end{equation}
we obtain
\begin{align}
& e^{im\phi}\sum_{n=-\infty}^{\infty}
(-i)^nJ_n(q\rho)\int_0^{2\pi}e^{i(m+n)\theta}\,d\theta.
\end{align}

Only the term $n=-m$ survives. Using $J_{-m}(x)=(-1)^m J_m(x)$ and the identity
$(-i)^{-m}(-1)^m=(-i)^m$, the integral evaluates to
\begin{equation}
2\pi (-i)^m e^{im\phi}J_m(q\rho).
\end{equation}
\end{proof}

\begin{theorem}[Fourier representation of the Fourier-Bessel wavelet]
\label{thm:fourier-wavelet}
Under the convention in \cref{eq:fourier-convention-2d}, the Fourier transform of \cref{eq:fb-wavelet} is
\begin{align}
\widehat{\psi}_{m,k}(q,\phi)
&=
(-i)^m e^{im\phi}N_{m,k}
\Bigg[\sigma^2e^{-\frac{\sigma^2}{2}(\lambda_{m,k}^2+q^2)}
I_m(\sigma^2\lambda_{m,k}q)
-K_{m,k}\sigma^2e^{-\sigma^2q^2/2}\Bigg].
\label{eq:fourier-wavelet}
\end{align}
\end{theorem}

\begin{proof}
Substituting the polar coordinates into the Fourier transform and using the Jacobian $\rho$ gives
\begin{align}
\widehat{\psi}_{m,k}(q,\phi)
&=\int_0^\infty\int_0^{2\pi}N_{m,k}e^{-\rho^2/(2\sigma^2)}\left[
J_m(\lambda_{m,k}\rho)-K_{m,k}
\right]e^{im\varphi}
\cdot e^{-iq\rho\cos(\varphi-\phi)}\rho\,d\varphi\,d\rho.
\end{align}

Applying \cref{lem:angular-integral} gives
\begin{align}
\widehat{\psi}_{m,k}(q,\phi)
&=(-i)^m e^{im\phi}N_{m,k}
\int_0^\infty e^{-\rho^2/(2\sigma^2)}
\left[J_m(\lambda_{m,k}\rho)-K_{m,k}\right]
J_m(q\rho)\rho\,d\rho.
\end{align}

The first radial term is Weber's second exponential integral,
\begin{align}
&\int_0^\infty
e^{-\rho^2/(2\sigma^2)}
J_m(\lambda_{m,k}\rho)
J_m(q\rho)
\rho\,d\rho
\nonumber\\
&\qquad=
\sigma^2
e^{-\frac{\sigma^2}{2}
(\lambda_{m,k}^2+q^2)}
I_m(\sigma^2\lambda_{m,k}q).
\end{align}

For the correction term, the remaining radial integral is
\begin{equation}
\int_0^\infty e^{-\rho^2/(2\sigma^2)}J_m(q\rho)\rho\,d\rho.
\end{equation}

For the $m=0$ correction used in the present construction, Weber's first
exponential integral gives
\begin{equation}
\int_0^\infty e^{-\rho^2/(2\sigma^2)}J_0(q\rho)\rho\,d\rho
=\sigma^2e^{-\sigma^2q^2/2}.
\end{equation}

Substituting the radial integrals gives \cref{eq:fourier-wavelet}.
\end{proof}

\begin{remark}[Separation of angular and radial structure]
\label{rem:fourier-separation}
The Fourier representation separates naturally into the angular factor
\begin{equation}
(-i)^m e^{im\phi}
\end{equation}
and a radial response depending only on $q$. Thus $m$ controls angular selectivity, while $\lambda_{m,k}$ controls the radial frequency location.
\end{remark}

\begin{figure}[h]
    \centering
    \includegraphics[width=0.7\linewidth]{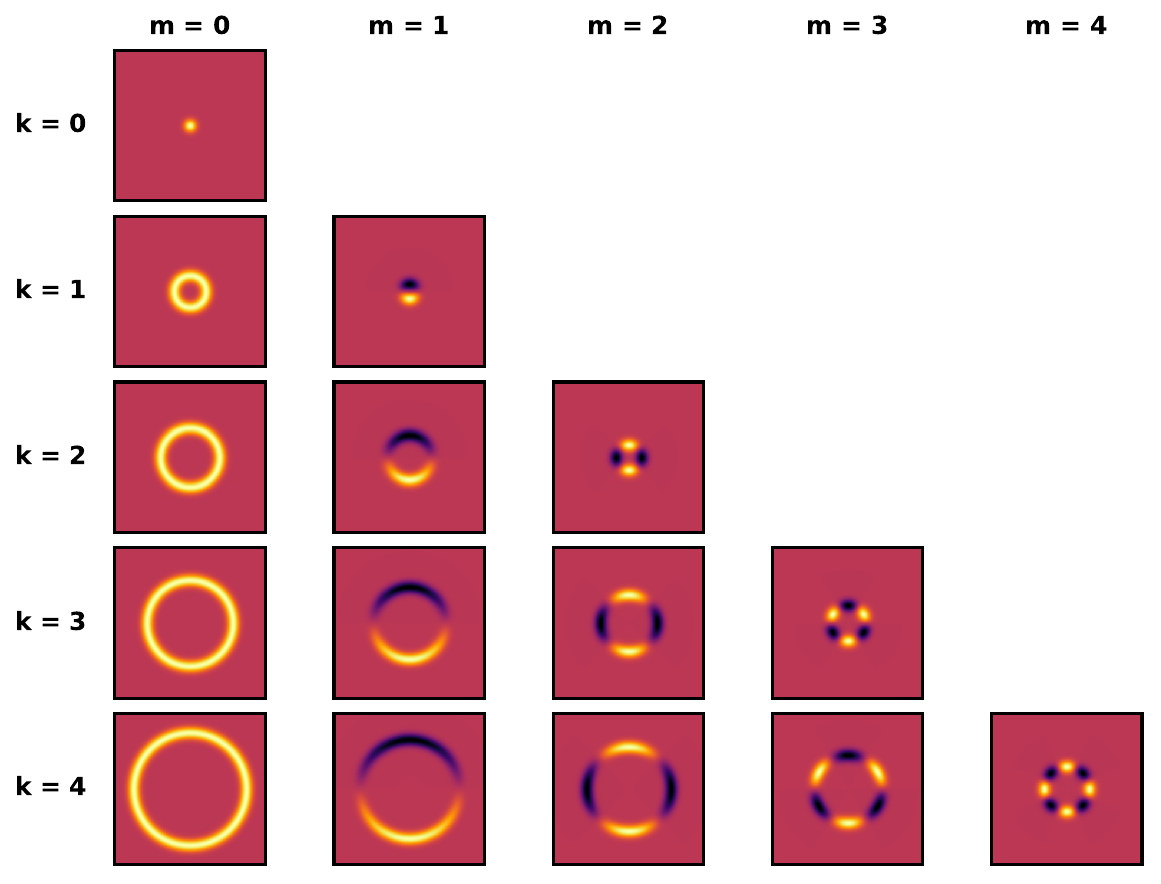}
    \caption{Equivalent Fourier-domain filter bank.}
    \label{fig:fourierbank}
\end{figure}

\begin{figure}[h]
    \centering
    \includegraphics[width=1\linewidth]{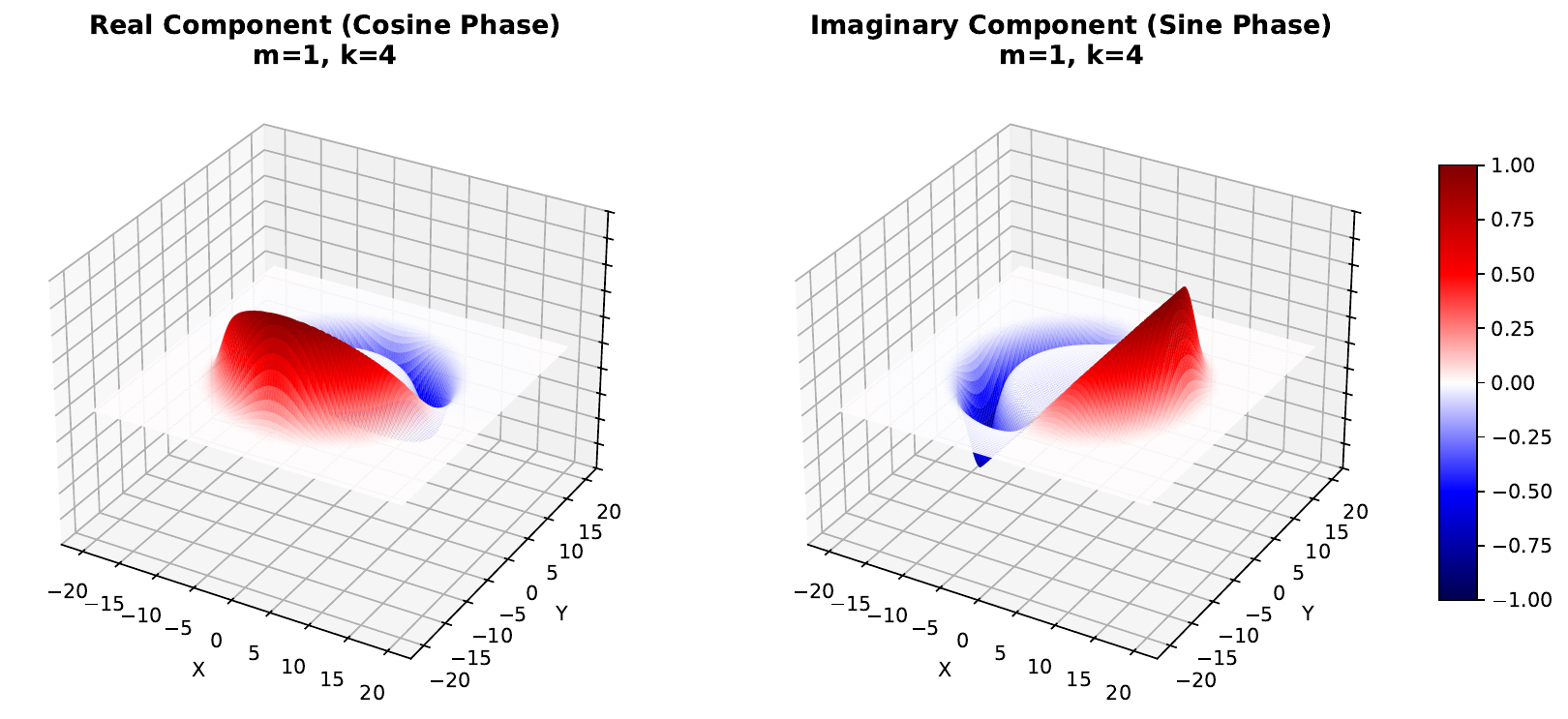}
    \caption{3D rendering of the real and imaginary parts of the
    Fourier-domain wavelet with $m=1$ and $k=4$.}
    \label{fig:3d_spatial_fourier}
\end{figure}

\FloatBarrier
\section{Frequency Tiling and Eigenvalue Spacing}

The use of the Neumann eigenvalues as the radial frequency parameter is
motivated by their approximately uniform spacing. From
\cref{eq:mcmahon},
\begin{equation}
\lambda_{m,k}\approx\beta_{m,k}-\frac{4m^2+3}{8\beta_{m,k}},
\end{equation}
with
\begin{equation}
\beta_{m,k}=\left(k+\frac{m}{2}-\frac{3}{4}\right)\pi.
\end{equation}

Therefore, the leading-order spacing is
\begin{equation}
\lambda_{m,k+1}-\lambda_{m,k}\longrightarrow\pi
\qquad\text{as }k\to\infty.
\label{eq:eigenvalue-spacing}
\end{equation}

\cref{tab:eigen_spacing} illustrates this convergence for the first few angular orders and root indices.
For higher angular orders, convergence to the asymptotic spacing is slower. This is consistent with the $m$-dependent correction term in \cref{eq:mcmahon}, whose numerator grows quadratically with $m$.

\begin{definition}[Frame] 
A filter bank consisting of a low-pass filter $\phi$ and a family of wavelet filters $\psi$ forms a frame for the signal space if there exist constants $0 < A \leq B < \infty$ such that:

\begin{equation}
    A \leq |\widehat{\phi}(\omega)|^2 + \sum_{i=1}^{\infty} |\widehat{\psi}_i(\omega)|^2 \leq B
\end{equation}

If $A=B$, the filter bank constitutes a tight frame, ensuring energy conservation (\cref{thm:plancherel}) and numerically stable, perfect reconstruction. In practice we aim to minimise the ratio with $A \approx B$.
\end{definition}

\begin{table}[h]
\centering
\begin{tabular}{@{} c c c c c c r r @{}}
\toprule
$m$ & $k$ & $k+1$ & $\lambda_k$ & $\lambda_{k+1}$ & Spacing
& \multicolumn{1}{c}{Error} & \multicolumn{1}{c}{Rel. Error} \\
&&&&&& \multicolumn{1}{c}{$(\Delta\lambda-\pi)$}
& \multicolumn{1}{c}{(\%)} \\
\midrule
0 & 1 & 2 & 3.8317 & 7.0156 & 3.1839 & +0.0423 & +1.35 \\
0 & 2 & 3 & 7.0156 & 10.1735 & 3.1579 & +0.0163 & +0.52 \\
0 & 3 & 4 & 10.1735 & 13.3237 & 3.1502 & +0.0086 & +0.27 \\
0 & 4 & 5 & 13.3237 & 16.4706 & 3.1469 & +0.0053 & +0.17 \\
0 & 5 & 6 & 16.4706 & 19.6159 & 3.1452 & +0.0036 & +0.12 \\
\addlinespace
1 & 1 & 2 & 1.8412 & 5.3314 & 3.4903 & +0.3487 & +11.10 \\
1 & 2 & 3 & 5.3314 & 8.5363 & 3.2049 & +0.0633 & +2.01 \\
1 & 3 & 4 & 8.5363 & 11.7060 & 3.1697 & +0.0281 & +0.89 \\
1 & 4 & 5 & 11.7060 & 14.8636 & 3.1576 & +0.0160 & +0.51 \\
1 & 5 & 6 & 14.8636 & 18.0155 & 3.1519 & +0.0103 & +0.33 \\
\addlinespace
2 & 1 & 2 & 3.0542 & 6.7061 & 3.6519 & +0.5103 & +16.24 \\
2 & 2 & 3 & 6.7061 & 9.9695 & 3.2633 & +0.1217 & +3.88 \\
2 & 3 & 4 & 9.9695 & 13.1704 & 3.2009 & +0.0593 & +1.89 \\
2 & 4 & 5 & 13.1704 & 16.3475 & 3.1772 & +0.0356 & +1.13 \\
2 & 5 & 6 & 16.3475 & 19.5129 & 3.1654 & +0.0238 & +0.76 \\
\addlinespace
3 & 1 & 2 & 4.2012 & 8.0152 & 3.8140 & +0.6725 & +21.40 \\
3 & 2 & 3 & 8.0152 & 11.3459 & 3.3307 & +0.1891 & +6.02 \\
3 & 3 & 4 & 11.3459 & 14.5858 & 3.2399 & +0.0983 & +3.13 \\
3 & 4 & 5 & 14.5858 & 17.7887 & 3.2029 & +0.0613 & +1.95 \\
3 & 5 & 6 & 17.7887 & 20.9725 & 3.1837 & +0.0421 & +1.34 \\
\addlinespace
4 & 1 & 2 & 5.3176 & 9.2824 & 3.9648 & +0.8233 & +26.20 \\
4 & 2 & 3 & 9.2824 & 12.6819 & 3.3995 & +0.2579 & +8.21 \\
4 & 3 & 4 & 12.6819 & 15.9641 & 3.2822 & +0.1406 & +4.48 \\
4 & 4 & 5 & 15.9641 & 19.1960 & 3.2319 & +0.0903 & +2.88 \\
4 & 5 & 6 & 19.1960 & 22.4010 & 3.2050 & +0.0634 & +2.02 \\
\bottomrule
\end{tabular}
\caption{Convergence of consecutive Neumann eigenvalue spacings toward
$\pi$ for angular orders $m=0,\ldots,4$.}
\label{tab:eigen_spacing}
\end{table}

\newpage
In \cref{fig:main_figure_combined}, we explore the frame bounds ratio $B/A$ to evaluate the wavelets. This experiment is not intended as a proof of frame bounds. It is included only to illustrate why proposed frequency organisation may merit further investigation. Note that the lower bound is evaluated at $0.75\pi$ to avoid the dividing by 0 when discrete wavelets naturally decay at the edge. Both wavelet families were evaluated on the same resolution, image size, variance and peak normalisation. Furthermore, while Solid Harmonics produce $J\times L$ wavelets, Fourier-Bessel wavelets, due to the pyramidal constraint, create $\sum_{m=0}^MK-m$ filters.

Across the tested parameter range, the Fourier-Bessel banks exhibit lower coverage ripple than the corresponding Solid Harmonic banks. This behaviour is consistent with the near linear spacing of the radial frequencies, which distributes the filters more uniformly. By contrast, the dyadic organisation of the Solid Harmonic filters places greater emphasis on lower frequencies and progressively wider spacing at higher frequencies. This difference should not be interpreted as evidence that linear spacing is preferable: the frequency weighting induced by dyadic scaling is an important feature of conventional wavelet constructions and can be desirable for tasks such as image classification, where greater emphasis on lower frequencies may contribute to robustness to small perturbations. Rather, the experiment suggests that linear frequency organisation is an interesting alternative when more uniform frequency representation is desired.

\begin{figure}[h]
    \centering
    \begin{subfigure}[t]{0.48\textwidth}
        \centering
        \includegraphics[width=\textwidth]{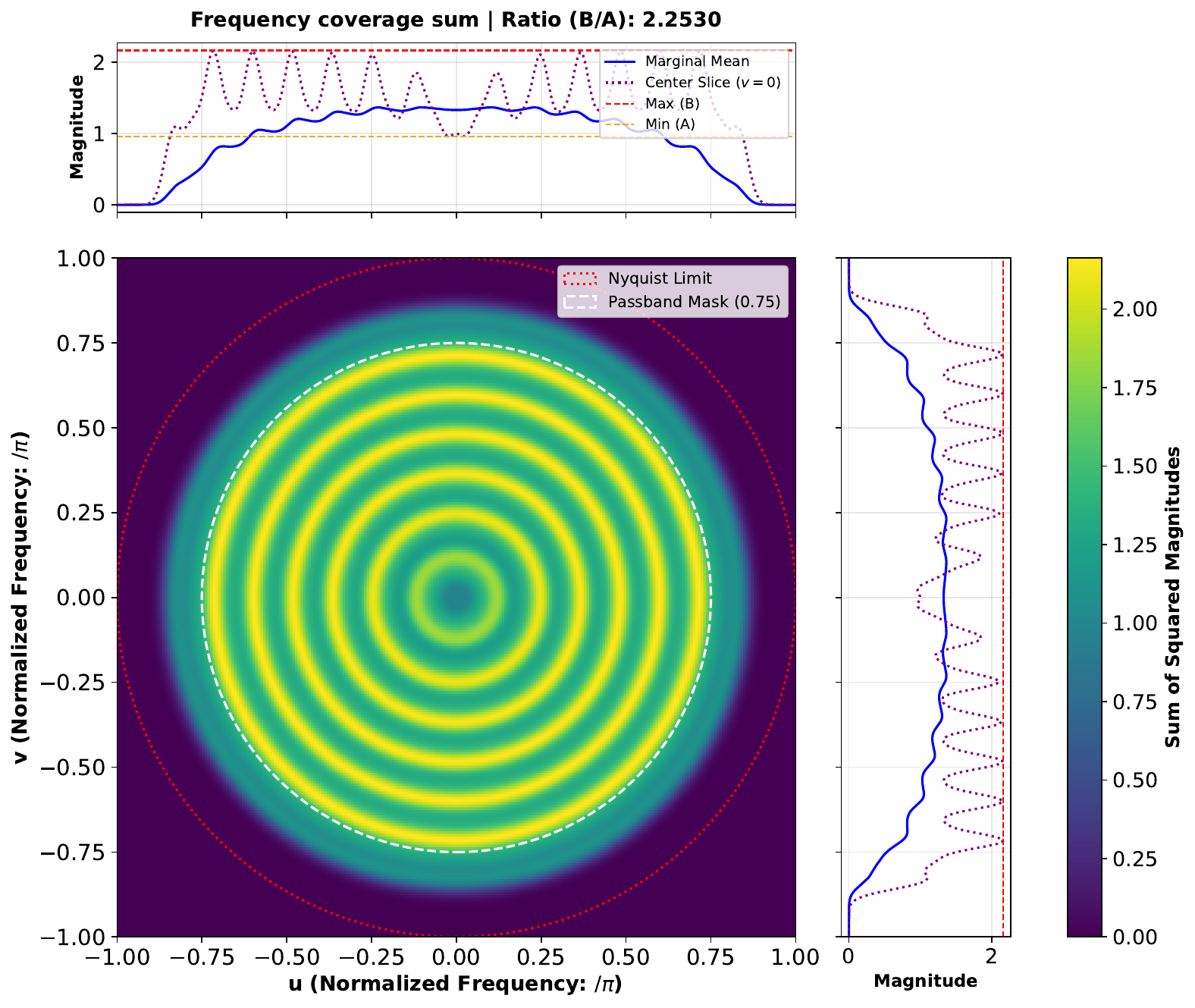}
        \caption{Frequency coverage sum for the Fourier-Bessel bank $m=3, k=8, \sigma=1$. The linear spacing is clearly visible with little overlap, achieving a ratio of $2.2530$.}
        \label{fig:paley_bessel}
    \end{subfigure}
    \hfill
    \begin{subfigure}[t]{0.48\textwidth}
        \centering
        \includegraphics[width=\textwidth]{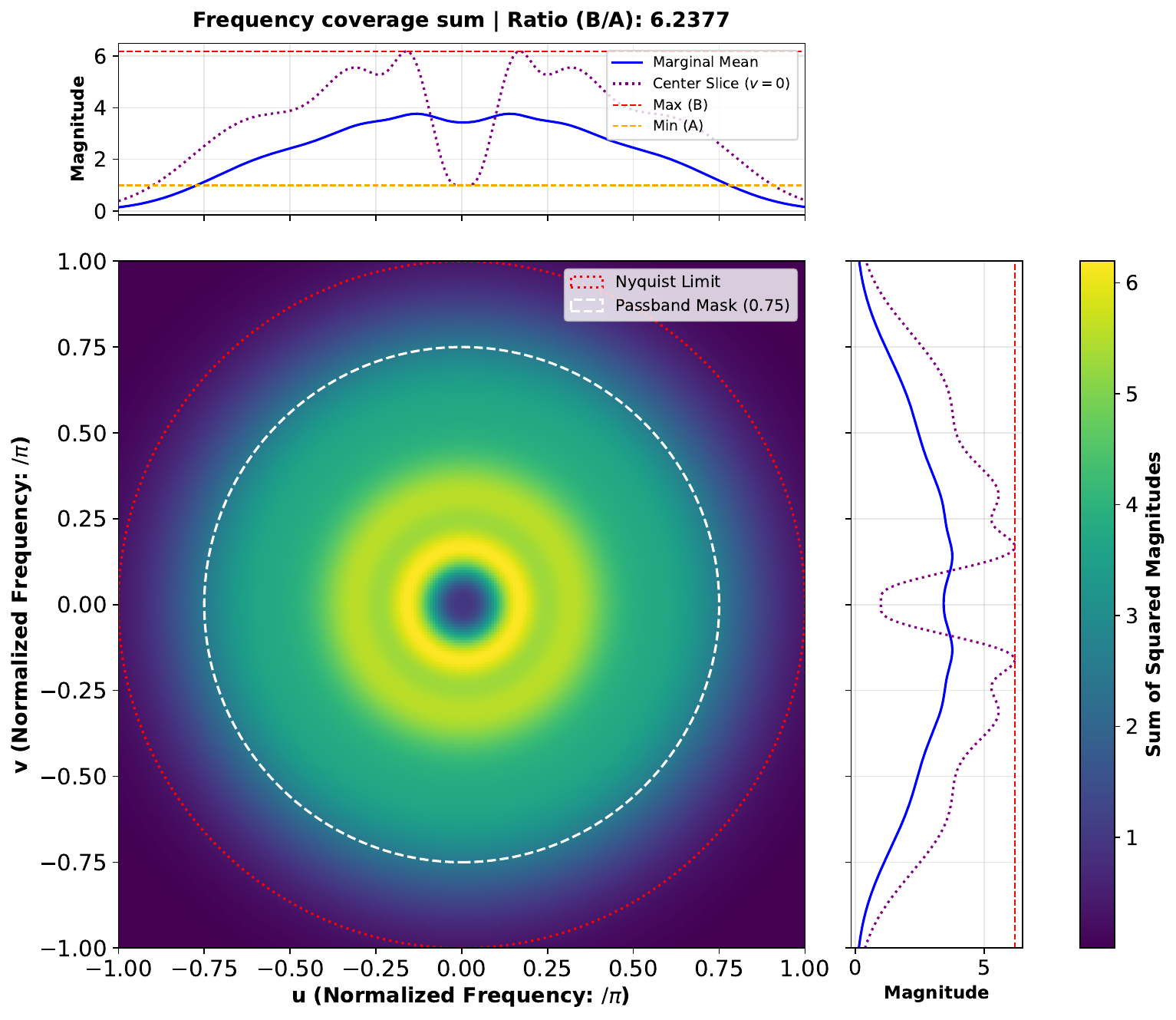}
        \caption{Frequency coverage sum for the Solid Harmonic bank $J=3, L=5, \sigma=1$. Here, dyadic scaling causes wider coverage gaps and overlaps (ratio of $6.2377$ in this example).}
        \label{fig:sh_bessel}
    \end{subfigure}

    \vspace{1em}

    \begin{subfigure}[t]{0.48\textwidth}
        \centering
        \includegraphics[width=\textwidth]{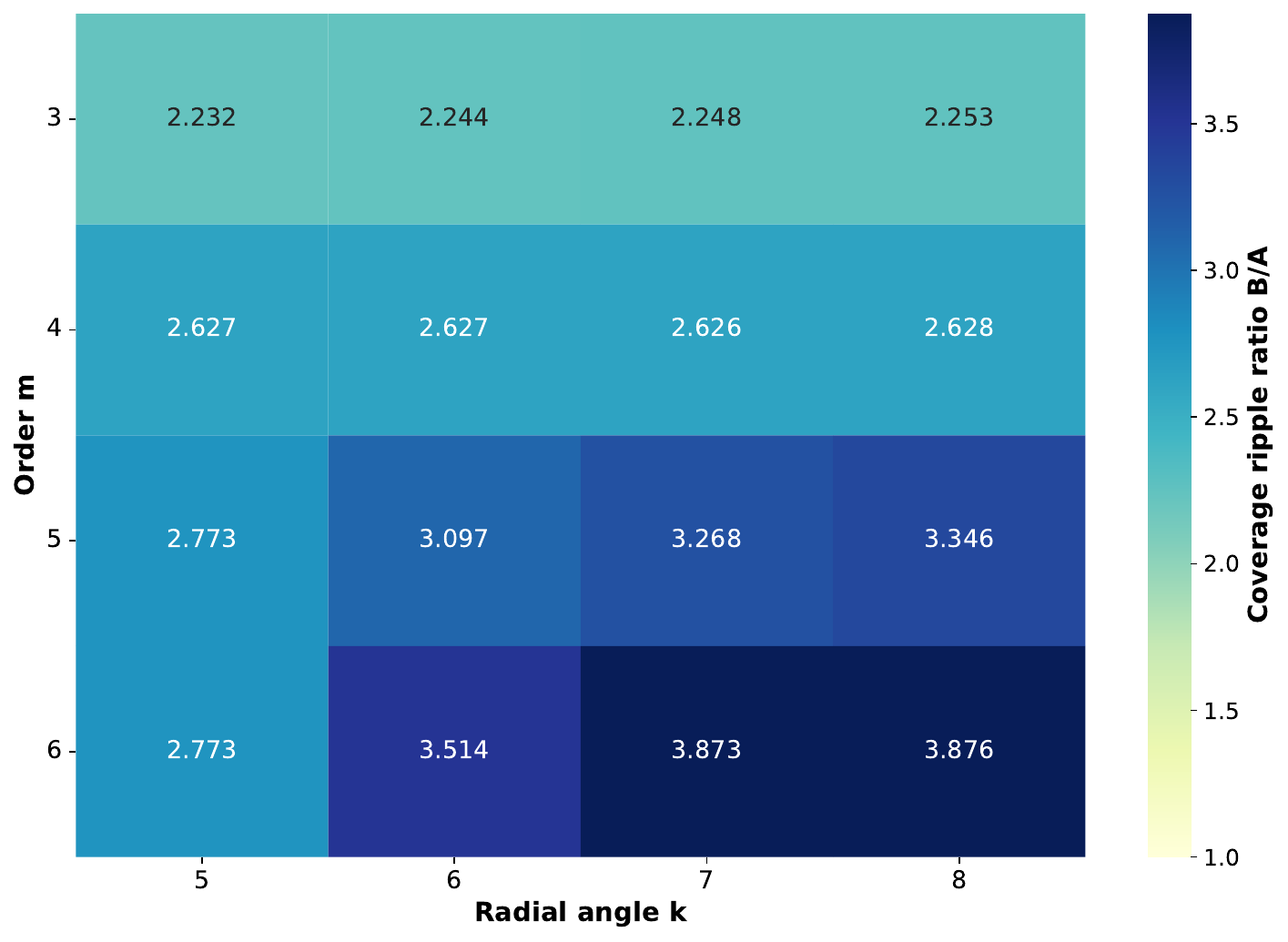}
        \caption{Fourier-Bessel frequency coverage parameter search. Ratio values remain consistent, naturally beginning to increase for larger orders $m$ as the bank begins to cover frequencies beyond the Nyquist limit, raising the lower bound at the cut off.}
        \label{fig:grid_bessel}
    \end{subfigure}
    \hfill
    \begin{subfigure}[t]{0.48\textwidth}
        \centering
        \includegraphics[width=\textwidth]{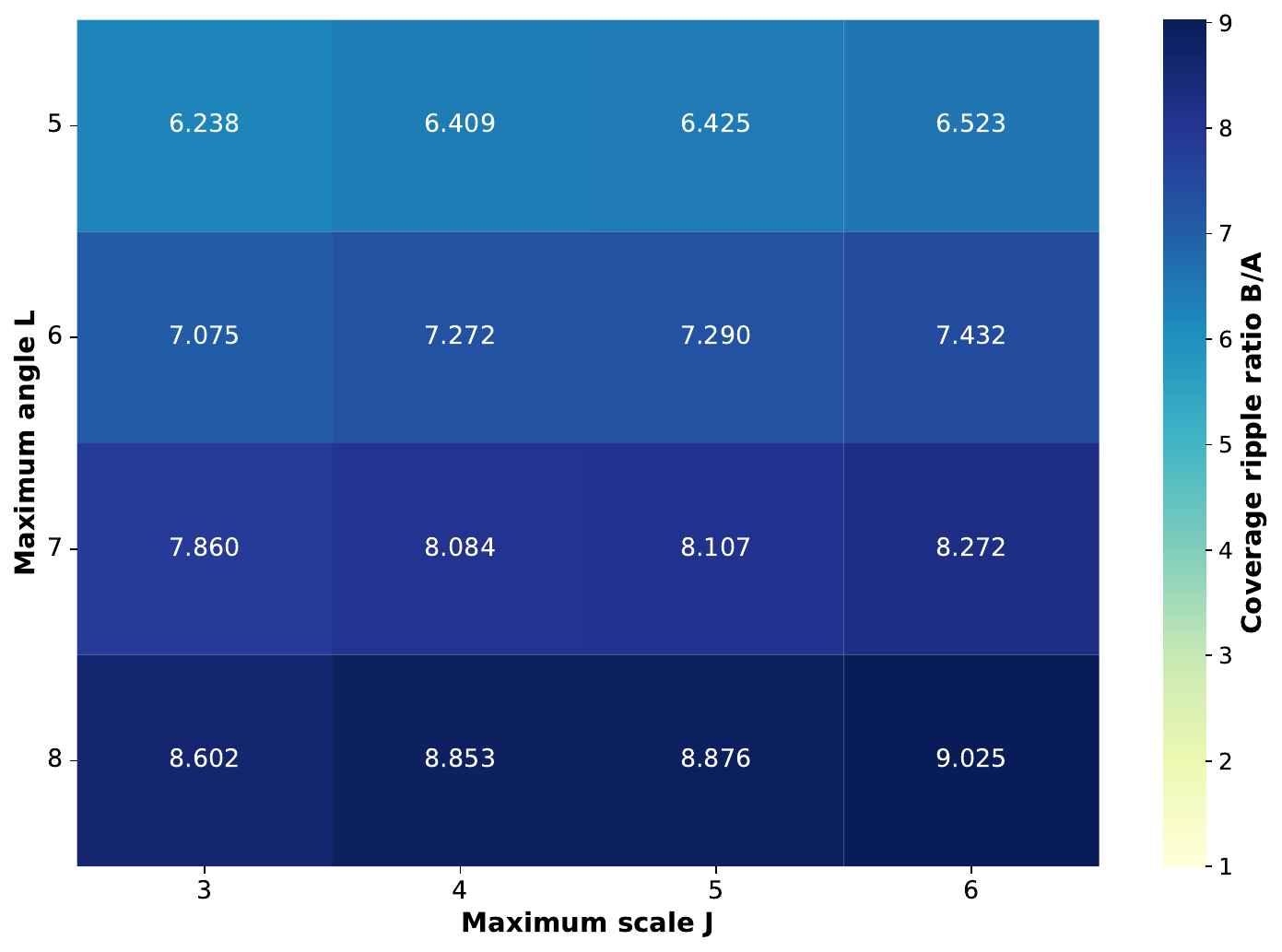}
        \caption{Solid Harmonic frequency coverage parameter search. Values follow a similar pattern, increasing as the bank exceeds the boundary, but the dyadic scaling causes a more rapid increase of the ratio.}
        \label{fig:grid_sh}
    \end{subfigure}

    \caption{Frequency coverage comparison between Fourier-Bessel wavelets and Solid Harmonics. Top row: Frequency coverage sums for a single representative bank of each family, showing the linear-spacing tiling of the Fourier-Bessel bank (a) against the dyadic-scaling gaps and overlaps of the Solid Harmonic bank (b). Bottom row: grid search over bank parameters, showing the frame ripple ratio $B/A$ for the Fourier-Bessel bank (c) and the Solid Harmonic bank (d); the Fourier-Bessel ratios stay consistently lower and grow more slowly than the Solid Harmonic ratios across the tested parameter range.}
    \label{fig:main_figure_combined}
\end{figure}

\FloatBarrier
\section{Conclusion}

These notes have developed the mathematical foundations and construction of Fourier-Bessel wavelets. Starting from the Bessel differential equation, we derived the Fourier-Bessel disk harmonics as solutions of the Helmholtz equation subject to a Neumann boundary condition. The resulting eigenvalues provide a natural radial frequency parameter whose asymptotic spacing approaches $\pi$.

We then constructed a wavelet family by applying a Gaussian spatial envelope to the Fourier-Bessel basis and introducing a zero-mean correction for the zeroth angular order. The corresponding $L^2$ normalisation constants were derived using Weber's exponential integrals, while a peak normalisation based on the radial Fourier response was developed for applications requiring consistent frequency-domain amplitudes.

Finally, we derived a closed-form Fourier-domain representation of the wavelets. This representation separates naturally into angular and radial components and provides a direct description of the frequency
response of each wavelet.

The preliminary frequency coverage provides evidence consistent with the motivation for the construction: in the configurations tested, the Fourier-Bessel bank exhibits a flatter frequency coverage profile relative to the corresponding Solid Harmonic bank. This experiment is intentionally limited and are not intended to establish general performance improvements.

The approximately linear radial frequency spacing should therefore be viewed as a complementary alternative to dyadic scaling rather than a replacement for it. Dyadic scaling remains central to wavelet theory and provides important theoretical and practical properties that have not been established for the present construction. The motivation for the Fourier-Bessel approach is instead to explore a different allocation of frequency resolution, which may be advantageous in reconstruction oriented settings where approximately uniform frequency representation is desirable. Determining the classes of tasks for which either frequency organisation is preferable is an open question.

The purpose of these notes is primarily theoretical and pedagogical. They are intended to provide a detailed mathematical reference for the construction rather than to constitute a comprehensive empirical evaluation of the resulting wavelet family.

The accompanying \texttt{fbscatnet} library implements the construction described throughout these notes and reproduces the figures presented here. A natural next step is to evaluate the resulting wavelets empirically within scattering networks and to compare their performance with existing wavelet constructions across relevant downstream
applications.

\newpage
\FloatBarrier
\printbibliography

\end{document}